\documentclass[graybox,square]{svmono}

\usepackage{amsmath,amssymb,amsfonts}
\usepackage{type1cm}        
\usepackage{nicefrac}
\usepackage{graphicx}        
\usepackage[none]{hyphenat}                           
\usepackage[bottom]{footmisc}

\usepackage{newtxtext}       %
\usepackage[varvw]{newtxmath}       

\usepackage{hyperref}
\usepackage{cprotect}

\makeindex             

\def\HH{{\mathbb H}}
\def\RR{{\mathbb R}}

\def\ZZ{{\mathbb Z}}
\def\SS{{\mathbb S}}
\def\x{{\bf x}}

\def\y{{\bf y}}

\def\be{\begin{equation}}
\def\ee{\end{equation}}
\def\bea{\begin{eqnarray}}
\def\eea{\end{eqnarray}}
\def\eref#1{(\ref{#1})}
\def\disp{\displaystyle}
\def\tn{|\!|\!|}
\def\XX{{\mathbb X}}
\def\BB{{\mathbb B}}

\def\ls{\lesssim}
\def\gs{\gtrsim}
\newcommand\norm[1]{\left|\left|#1\right|\right|}
\newcommand\abs[1]{\left|#1\right|}

\spnewtheorem{uda}{Example}{\bf}{}

\begin{document}
	
	\guidelinedefn%

	\title*{Local transfer learning from one data space to another
		}
\author{H.~N.~Mhaskar\thanks{Institute of Mathematical Sciences, Claremont Graduate University, Claremont, CA 91711. The research of HNM was supported in part by ARO grant W911NF2110218 and NSF DMS grant 2012355. \textsf{email:} hrushikesh.mhaskar@cgu.edu.}~and~Ryan~O'Dowd\thanks{Institute of Mathematical Sciences, Claremont Graduate University, Claremont CA 91711. \textsf{email:} ryan.o'dowd@cgu.edu.}}
\maketitle

\abstract{
A fundamental problem in manifold learning is to approximate a functional relationship in a data chosen randomly from a probability distribution supported on a low dimensional sub-manifold of a high dimensional ambient Euclidean space. 
The manifold is essentially defined by the data set itself and, typically, designed so that the data is dense on the manifold in some sense.
The notion of a data space is an abstraction of a manifold encapsulating the essential properties that allow for function approximation.
The problem of transfer learning (meta-learning) is to use the learning of a function on one data set to learn a similar function on a new data set.
In terms of function approximation, this means lifting a function on one data space (the base data space) to another (the target data space). 
This viewpoint enables us to connect some inverse problems in applied mathematics (such as inverse Radon transform) with transfer learning.
In this paper we examine the question of such lifting when the data is assumed to be known only on a part of the base data space.
We are interested in determining subsets of the target data space on which the lifting can be defined, and how the local smoothness of the function and its lifting are related.}

\section{Introduction}\label{bhag:introduction}

A fundamental problem in machine learning is the following. 
A data of the form $\{(x_j,y_j)\}$ is given, assumed to be sampled from an unknown probability distribution.
The goal is to approximate the function $f(x)=\mathbb{E}(y|x)$ from the data.
Typically, the points $x_j$ belong to an ambient Euclidean space of a very high dimension, leading to the so called  curse of dimensionality.
One of the strategies to counter this ``curse'' is to assume the \emph{manifold hypothesis}; i.e., assume that the  points $x_j$ are located on an unknown low dimensional submanifold of the ambient space. 
Examples of some well known techniques in this direction, dimensionality reduction in particular, are  Isomaps \cite{tenenbaum2000global}, maximum variance unfolding (MVU) (also called
 semidefinite programming (SDP)) \cite{weinberger2005nonlinear}, locally linear embedding
 (LLE) \cite{roweis2000nonlinear}, local tangent space alignment method (LTSA) \cite{zhang2004principal}, Laplacian eigenmaps (Leigs) \cite{belkin2003laplacian}, Hessian locally
 linear embedding (HLLE) \cite{david2003hessian},  diffusion maps (Dmaps) \cite{coifmanlafondiffusion}, and randomized anisotropic transform \cite{chuiwang2010}. 
 A recent survey of these methods is given by Chui and Wang in \cite{chuidimred2015}.  
 An excellent introduction to the subject of diffusion geometry can be found in the special issue \cite{achaspissue} of Applied and Computational Harmonic Analysis, 2006. 
 The application areas are too numerous to mention exhaustively. 
 They include, for example, document analysis \cite{coifmanmauro2006}, 
 face recognition \cite{niyogiface, ageface2011, chuiwang2010}, hyperspectral imaging \cite{chuihyper},  semi-supervised learning \cite{niyogi1, niyogi2}, image processing \cite{donoho2005image, arjuna1},  cataloguing of galaxies \cite{donoho2002multiscale}, and social networking \cite{bertozzicommunity}.
 
A good deal of research in the theory of manifold learning deals with the problem of understanding the geometry of the data defined manifold. 
For example,  it is shown in \cite{jones2010universal, jones2008parameter} that an atlas on the unknown manifold can be defined in terms of the heat kernel corresponding to the Laplace-Beltrami operator on the manifold.
Other constructions  of the atlas are given in \cite{chui_deep, shaham2018provable, schmidt2019deep} with applications to the study of deep networks.
Function approximation on manifolds based on \emph{scattered data} (i.e., data points $x_j$ whose locations are not prescribed analytically) has been studied in detail in many papers, starting with \cite{mauropap}, e.g., \cite{frankbern, modlpmz, eignet, compbio, heatkernframe, mhaskar2020kernel}. 
A theory was applied successfully in \cite{mhas_sergei_maryke_diabetes2017} to construct  deep networks for predicting blood sugar levels based on continuous glucose monitoring devices. 

A fundamental role in this theory is played by the heat kernel on the manifold corresponding to an appropriate elliptic partial differential operator. 
In \cite{coifmanmauro2006, heatkernframe}, a muti-resolution analysis is constructed using the heat kernel. 
Another important tool is the theory of localized kernels based on the eigen-decomposition of the heat kernel. 
These were introduced in \cite{mauropap} based on certain assumptions on the spectral function and the property of finite speed of wave propagation. 
In the context of manifolds, this later property was proved in \cite{sikora2004riesz, frankbern} to be equivalent to the so called Gaussian upper bounds on the heat kernels. 
Although such bounds are studied in many contexts by many authors, e.g., \cite{grigoryan1995upper, grigor1997gaussian, davies1990heat, kordyukov1991p},  we could not locate a reference where such a bound was proved for a general smooth manifold. 
We have therefore supplied a proof in \cite{mhaskar2020kernel}.
In \cite[Theorem~4.3]{tauberian}, we have proved a very general recipe that yields localized kernels based on the Gaussian upper bound on the heat kernel in what we have termed a data defined space (or data space in some other papers). 

The problem of transfer learning (or meta-learning) involves learning the parameters of an approximation process based on one data set, and using this information to quickly learn the corresponding parameters on another data set, e.g., \cite{valeriyasmartphone, maskey2023transferability, maurer2013sparse}.
In the context of manifold learning, a data set (point cloud) determines a manifold, so that different data sets would correspond to different manifolds. 
In the context of data spaces, we can therefore interpret transfer learning as ``lifting'' a function from one data space (the \emph{base data space}) to another (the \emph{target data space}).
This viewpoint allows us to unify the topic of transfer learning with the study of some inverse problems in image/signal processing.
For example, the problem of synthetic aperture radar (SAR) imaging can be described in terms of an inverse Radon transform \cite{nolan2002synthetic,  cheney2009fundamentals, munson1983tomographic}. 
The domain and range of the Radon transform are different, and hence, the problem amounts to approximating the actual image on one domain based on observations of its Radon transform, which are located on a different domain.
Another application is in analyzing hyperspectral images changing with time \cite{coifmanhirn}. 
A similar problem arises in analyzing the progress of Alzheimer's disease from MRI images of the brain taken over time, where one is interested in the development of the cortical thickness as a function on the surface of the brain, a manifold which is changing over time \cite{kim2014multi}.

Motivated by these applications and the paper \cite{coifmanhirn} of Coifman and Hirn, we studied in \cite{tauberian} the question of lifting a function from one data space to another, when certain landmarks from one data space were identified with those on the other data space. 
For example, it is known \cite{lerch2005focal} that in spite of the changing brain, one can think of each brain to be parametrized by an inner sphere, and the cortical thickness at certain standard points based on this parametrization are important in the prognosis of the disease.
In \cite{tauberian} we investigated certain conditions on the two data spaces which allow the lifting of a function from one to the other, and analyzed the effect on the smoothness of the function as it is lifted.

In many applications, the data about the function is available only on a part of the base data space. 
The novel part of this paper is to investigate the following questions of interest: 
(1) determine on what subsets of the target data space the lifting is defined, and (2) how the local smoothness on the base data space translates into the local smoothness of the lifted function. In limited angle tomography, one observes the Radon transform on a limited part of a cylinder and needs to reconstruct the image as a function on a ball from this data. 
A rudimentary introduction to the subject is given in the book \cite{natterer2001mathematics} of Natterer.
We do not aim to solve  the limited angle tomography problem itself, but we will study in detail an example motivated by the singular value decomposition of the Radon transform, which involves two different systems of orthogonal polynomials on the interval $[-1,1]$.
The theory of transplantation theorems \cite{muckenhoupt1986transplantation} deals with the following problem. 
We are given the coefficients in the expansion of a function $f$ on $[-1,1]$ in terms of Jacobi polynomials with certain parameters (the base space expansion in our language), and use them as the coefficients in an expansion in terms of Jacobi polynomials with respect to a different set of parameters (the target space in our language). 
Under what conditions on $f$ and the parameters of the two Jacobi polynomial systems will the expansion in the target space converge and in which $L^p$ spaces?
While old fashioned, the topic appears to be of recent interest \cite{diaz2021discrete, arenas2019weighted}. 
We will illustrate our general theory by obtaining a localized transplantation theorem for uniform approximation.

In Section~\ref{bhag:singlespace}, we review certain important results in the context of a single \emph{data space} (our abstraction of a manifold). In particular, we present a characterization of local approximation of functions on such spaces.
In Section~\ref{bhag:jointspaces}, we review the notion of joint spaces (introduced under a different name in \cite{tauberian}).
The main new result of our paper is to study the lifting of a function from a subset (typically, a ball) on one data space to another. 
These results are discussed in Section~\ref{bhag:locapprox}. 
The proofs are given in Section~\ref{bhag:proofs}.
An essential ingredient in our constructions is the notion of localized kernels which, in turn, depend upon a Tauberian theorem. 
For the convenience of the reader, this theorem is presented in Appendix~\ref{bhag:tauberian}. 
Appendix~\ref{bhag:jacobi} lists some important properties of Jacobi polynomials which are required in our examples. 

\section{Data spaces}\label{bhag:singlespace}

As mentioned in the introduction, a good deal of research on manifold learning is devoted to the question of learning the geometry of the manifold. 
For the purpose of harmonic analysis and approximation theory on the manifold, we do not need the full strength of the differentiability structure on the manifold. 
Our own understanding of the correct hypotheses required to study these questions has evolved, resulting in a plethora of terminology such as data defined manifolds, admissible systems, data defined spaces, etc., culminating in our current understanding with the definition of a data space given in \cite{mhaskar2020kernel}. 
For the sake of simplicity, we will restrict our attention in this paper to the case of compact spaces. 
We do not expect any serious problems in extending the theory to the general case, except for a great deal of technical details.

Thus, the set up is the following.

We consider a compact metric measure space $\XX$ with metric $d$ and a probability measure $\mu^*$. We take $\{\lambda_k\}_{k=0}^\infty$ to be a non-decreasing sequence of real numbers with $\lambda_0=0$ and $\lambda_k\to\infty$ as $k\to\infty$, and $\{\phi_k\}_{k=0}^\infty$ to be an orthonormal set in $L^2(\mu^*)$.
We assume that each $\phi_k$ is continuous.
The elements of the space
\be\label{eq:diffpolyspace}
\Pi_n=\mathsf{span}\{\phi_k: \lambda_k <n\}
\ee
are called \emph{diffusion polynomials} (of order $<n$). 
We write $\disp\Pi_\infty=\bigcup_{n>0}\Pi_n$.
We introduce the following notation.
\be\label{eq:balldef}
\BB(x,r)=\{y\in \XX : d(x,y)\le r\},  \qquad x\in\XX, \ r>0.
\ee
If $A\subseteq \XX$ we define
\be\label{eq:setnbds}
\BB(A,r)=\bigcup_{x\in A}\BB(x,r).
\ee

With this set up, the definition of a compact data space is the following.

\begin{definition}\label{def:ddrdef}
The  tuple $\Xi=(\XX,d,\mu^*, \{\lambda_k\}_{k=0}^\infty, \{\phi_k\}_{k=0}^\infty)$ is called a \textbf{(compact) data space} if 
each of the following conditions is satisfied.
\begin{enumerate}
\item For each $x\in\XX$, $r>0$, $\mathbb{B}(x,r)$ is compact.
\item (\textbf{Ball measure condition}) There exist $q\ge 1$ and $\kappa>0$ with the following property: For each $x\in\XX$, $r>0$,
\be\label{eq:ballmeasurecond}
\mu^*(\mathbb{B}(x,r))=\mu^*\left(\{y\in\XX: d(x,y)<r\}\right)\le \kappa r^q.
\ee
(In particular, $\mu^*\left(\{y\in\XX: d(x,y)=r\}\right)=0$.)
\item (\textbf{Gaussian upper bound}) There exist $\kappa_1, \kappa_2>0$ such that for all $x, y\in\XX$, $0<t\le 1$,
\be\label{eq:gaussianbd}
\left|\sum_{k=0}^\infty \exp(-\lambda_k^2t)\phi_k(x)\phi_k(y)\right| \le \kappa_1t^{-q/2}\exp\left(-\kappa_2\frac{d(x,y)^2}{t}\right).
\ee
\end{enumerate}
We refer to $q$ as the \textbf{exponent} for $\Xi$.
\end{definition}

The primary example of a data space is, of course, a Riemannian manifold.

\begin{uda}\label{uda:manifold}
{\rm  Let $\XX$ be a smooth, compact, connected Riemannian manifold (without boundary), $d$ be the geodesic distance on $\XX$, $\mu^*$ be the Riemannian volume measure normalized to be a probability measure, $\{\lambda_k\}$ be the sequence of eigenvalues of the (negative) Laplace-Beltrami operator on $\XX$, and $\phi_k$ be the eigenfunction corresponding to the eigenvalue $\lambda_k$; in particular, $\phi_0\equiv 1$. 
We have proved in  \cite[Appendix~A]{mhaskar2020kernel} that the Gaussian upper bound is satisfied. 
Therefore, if the condition in Equation~\eqref{eq:ballmeasurecond} is satisfied, then $(\XX,d,\mu^*, 
\{\lambda_k\}_{k=0}^\infty, \{\phi_k\}_{k=0}^\infty)$ is a data space with exponent equal to the dimension of the manifold. 
\qed}
\end{uda}

\begin{remark}\label{rem:graph}
{\rm
In \cite{friedman2004wave}, Friedman and Tillich give a construction for an orthonormal system on a graph which leads to a finite speed of wave propagation. 
It is shown in \cite{frankbern} that this, in turn, implies the Gaussian upper bound. 
Therefore, it is an interesting question whether appropriate definitions of measures and distances can be defined on a graph to satisfy the assumptions of a data space.
\qed}
\end{remark}

\noindent\textbf{The constant convention.}
\emph{
In the sequel, $c, c_1,\cdots$ will denote generic positive constants depending only on the fixed quantities under discussion such as $\Xi$, $q$, $\kappa,\kappa_1,\kappa_2$,  the various smoothness parameters and the filters to be introduced. 
Their value may be different at different occurrences, even within a single formula.
The notation $A\ls B$ means $A\le cB$, $A\gs B$ means $B\ls A$ and $A\sim B$ means $A \ls B\ls A$.\qed
}

\begin{uda}\label{uda:jacobispace}
{\rm In this example, we let $\mathbb{X}=[0,\pi]$ and for $\theta_1,\theta_2
\in\mathbb{X}$ we simply define the distance as
\begin{equation}
d(\theta_1,\theta_2)=\abs{\theta_1-\theta_2}.
\end{equation}
We will consider the so-called \textit{trigonometric functions} \cite{nowak2011sharp}
\begin{equation}
\phi_n^{(\alpha,\beta)}(\theta)=(1-\cos\theta)^{\alpha/2+1/4}(1+\cos\theta)^{\beta/2+1/4}p_n^{(\alpha,\beta)}(\cos\theta),
\end{equation}
where $p_n^{(\alpha,\beta)}$ are orthonormalized Jacobi polynomials defined as in Appendix~\ref{bhag:jacobi} and $\alpha,\beta\geq -1/2$. We define
\begin{equation}
d\mu^*(\theta)=\frac{1}{\pi}d\theta.
\end{equation}
We see that a change of variables $x=\cos\theta$ in Equation~\eqref{eq:jacobiortho} results in the following orthogonality condition
\begin{equation}
\int_0^\pi \phi_n^{(\alpha,\beta)}(\theta)\phi_m^{(\alpha,\beta)}(\theta)d\theta=\delta_{n,m}.
\end{equation}
So our orthonormal set of functions with respect to $\mu^*$ will be $\{\sqrt{\pi}\phi_n^{(\alpha,\beta)}\}$. It was proven in \cite{nowak2011sharp} that with
\begin{equation}
\lambda_n=n+\frac{\alpha+\beta+1}{2},
\end{equation}
we have
\begin{equation}
\pi\sum_{n=0}^\infty \exp\left(-\lambda_{n}^2t\right)\phi_n^{(\alpha,\beta)}(\theta_1)\phi_n^{(\alpha,\beta)}(\theta_2)\lesssim t^{-1/2}\exp\left(-c\frac{d(\theta_1,\theta_2)^2}{t}\right), \quad \theta_1,\theta_2\in\mathbb{X}.
\end{equation}
In conclusion,
\begin{equation}
	\Xi=(\mathbb{X},d,\mu^*,\{\lambda_n\},\{\sqrt{\pi}\phi_n^{(\alpha,\beta)}\})
\end{equation}
is a data space with exponent $1$.
\qed}
\end{uda}

The following example illustrates how a manifold with boundary can be transformed into a closed manifold as in Example~\ref{uda:manifold}.
We will use the notation and facts from Appendix~\ref{bhag:jacobi} without always referring to them explicitly.
We adopt the notation 
\be\label{eq:spheredef}
\SS^q=\{\x\in\RR^{q+1}: |\x|=1\}, \qquad \SS^q_+=\{\x\in\SS^q : x_{q+1}\ge 0\}.
\ee

\begin{uda}\label{uda:pnjk}
{\rm
Let $\mu^*_q$ denote the volume measure of $\mathbb{S}^q$, normalized to be a probability measure. Let $\HH_{n}^{q}$ be the space of the restrictions to $\SS^{q}$ of homogeneous harmonic polynomials of degree $n$ on $q+1$ variables, and $\{Y_{n,k}\}_k$ be an orthonormal (with respect to $\mu^*_q$) basis for $\HH_{n}^{q}$.
The polynomials $Y_{n,k}$ are eigenfunctions of the Laplace-Beltrami operator on the manifold $\SS^q$ with eigenvalues $n(n+q-1)$. The geodesic distance between $\xi,\eta\in\mathbb{S}^q$ is $\arccos(\xi\cdot \eta)$, so the Gaussian upper bound for manifolds takes the form
\be\label{eq:sphgauss}
\sum_{n,k}\exp(-n(n+q-1)t)Y_{n,k}(\x)\overline{Y_{n,k}(\y)}\ls t^{-q/2}\exp\left(-c\frac{(\arccos(\x\cdot\y))^2}{t}\right).
\ee
As a result, $(\mathbb{S}^q,\arccos(\circ\cdot \circ),\mu^*_q,\{\lambda_{n}\}_n,\{Y_{n,k}\}_{n,k})$ is a data space with dimension $q$.

Now we consider 
$$
\XX=\BB^q=\{\x\in\RR^q: |\x|\le 1\}.
$$
We can identify $\BB^q$ with $\SS^q_+$ as follows. Any point $\x\in\BB^q$ has the form $\x=\omega\sin\theta$ for some   $\omega\in\SS^{q-1}$, $\theta\in [0,\pi/2]$. We write $\hat{\x}=(\omega\sin\theta, \cos\theta)\in\SS^q_+$. 
With this identification, $\SS^q_+$ is parameterized by $\BB^q$ and we define
\begin{equation}\label{eq:ballmeasure}
\begin{aligned}
d\mu^*(\x)=d\mu^*_q(\hat{\x})=&\frac{\operatorname{Vol}(\mathbb{B}^q)}{\operatorname{Vol}(\mathbb{S}^q_+)}(1-|\x|^2)^{-1/2}dm^*(\x)\\
=&\frac{\Gamma((q+1)/2)}{\sqrt{\pi}\Gamma(q/2+1)}(1-|\x|^2)^{-1/2}dm^*(\x),
\end{aligned}
\end{equation}
where $\mu^*_q$ is the probability volume measure on $\mathbb{S}^q_+$, and $m^*$ is the probability volume measure on $\mathbb{B}^q$.
It is also convenient to define the distance on $\BB^q$ by 
\begin{equation}
d(\x_1,\x_2)=\arccos(\hat{\x}_1\cdot\hat{\x}_2)=\arccos(\x_1\cdot\x_2 +\sqrt{1-|\x_1|^2}\sqrt{1-|\x_2|^2}).
\end{equation}
All spherical harmonics of degree $2n$ are even functions on $\mathbb{S}^q$. So with the identification of measures as above, one can represent the even spherical harmonics as an orthonormal system of functions on $\mathbb{B}^q$. That is, by defining
\be
P_{2n,k}(\x)=\sqrt{2}Y_{2n,k}(\hat{\x}),
\ee
we have
\be\label{eq:ballortho}
\begin{aligned}
\int_{\BB^q} P_{2n,k}(\x)\overline{P_{2n',k'}(\x)}d\mu^*(\x) =&2\int_{\mathbb{S}_+^q}Y_{2n,k}(\hat\x)\overline{Y_{2n',k'}(\hat\x)}d\mu^*_q(\hat\x)\\
=&\int_{\mathbb{S}^q}Y_{2n,k}(\xi)\overline{Y_{2n',k'}(\xi)}d\mu^*_q(\xi)\\
=&\delta_{(n,k), (n',k')}.
\end{aligned}
\ee
To show the Gaussian upper bound for $Y_{2n,k}$ on $\mathbb{B}^q$, we first see that in view of the addition formula \eqref{eq:addformula} and \eqref{eq:evenjacobi}, we deduce
\begin{equation}\begin{aligned}
&\sum_{k=1}^{\operatorname{dim}(\mathbb{H}_{2n}^q)}P_{2n,k}(\x)\overline{P_{2n,k}(\y)}\\
=&\sum_{k=1}^{\operatorname{dim}(\mathbb{H}_{2n}^q)}Y_{2n,k}(\hat\x)\overline{Y_{2n,k}(\hat\y)}\\
=&\frac{\omega_q}{\omega_{q-1}} p_{2n}^{(q/2-1,q/2-1)}(1)p_{2n}^{(q/2-1,q/2-1)}(\hat\x\cdot\hat\y)\\
=&\frac{\omega_q}{\omega_{q-1}}2^{(q-1)/2}p_n^{(q/2-1,-1/2)}(1)p_n^{(q/2-1,-1/2)}(\cos(2\arccos(\hat\x\cdot \hat\y))).
\end{aligned}
\end{equation}
In light of Equation~\eqref{eq:jacobidifeq} we define
\begin{equation}
\lambda_{n}=\sqrt{n(n+q/2-1/2)},
\end{equation}
which is conveniently not dependent upon $k$. Using \eqref{eq:specialjacobigauss},  we see that for $t>0$
\begin{equation}\label{eq:pgauss}
\begin{aligned}
&\sum_{n=0}^\infty\sum_{k=1}^{\operatorname{dim}(\mathbb{H}_{2n}^q)}\exp\left(-\lambda_{n}^2t\right)P_{2n,k}(\mathbf{x})\overline{P_{2n,k}(\y)}\\
\sim& \sum_{n=0}^\infty\exp(-n(n+q/2-1/2)t)p_{n}^{(q/2-1,-1/2)}(1)p_{n}^{(q/2-1,-1/2)}(\cos(2\arccos(\hat\x\cdot \hat\y)))\\
\lesssim& t^{-q/2}\left(-4c\frac{\arccos(\hat{\x}_1\cdot \hat{\x}_2)^2}{t}\right).
\end{aligned}
\end{equation}
Therefore, $(\mathbb{B}^q,d,\mu^*,\{\lambda_n\}_n,\{P_{2n,k}\}_{n,k})$ is a data space with exponent $q$.
\qed}
\end{uda}

In this section, we will assume $\Xi$ to be a fixed data space and omit its mention from the notations. 
We will mention it later in other parts of the paper in order to avoid confusion.
Next, we define smoothness classes of functions on $\XX$. 
In the absence of any differentiability structure, we do this in a manner that is customary in approximation theory. We define first the \emph{degree of approximation} of a function $f\in L^p(\mu^*)$ by
\be\label{eq:degapprox}
E_n(p,f)=\min_{P\in \Pi_n}\|f-P\|_{p,\mu^*}, \qquad n>0,  1\le p \le \infty,\ f\in L^p(\mu^*).
\ee
We find it convenient to denote by $X^p$ the space $\{f\in L^p(\mu^*) : \disp\lim_{n\to\infty}E_n(p,f)=0\}$; e.g., in the manifold case, $X^p=L^p(\mu^*)$ if $1\le p<\infty$ and $X^\infty=C(\XX)$.
In the case of Example~\ref{uda:pnjk}, we need to restrict ourselves to even functions.

\begin{definition}\label{def:sobolev}
Let $1\le p\le \infty$, $\gamma>0$. \\
{\rm (a)} For $f\in X^p$, we define
\be\label{sobnorm}
\|f\|_{W_{\gamma,p}}=\|f\|_{p,\mu^*}+\sup_{n>0}n^\gamma E_n(p,f),
\ee
and note that
\be\label{sobnormuseful}
\|f\|_{W_{\gamma,p}}\sim \|f\|_{p,\mu^*}+\sup_{n\in\ZZ_+}2^{n\gamma}E_{2^n}(p,f).
\ee
The space $W_{\gamma,p}$ comprises all $f$ for which $\|f\|_{W_{\gamma,p}} <\infty$.\\
{\rm (b)}
We write $C^\infty=\displaystyle\bigcap_{\gamma>0}W_{\gamma,\infty}$.
If $B$ is a ball in $\XX$, $C^\infty(B)$ comprises functions $f\in C^\infty$ which are supported on $B$.\\
{\rm (c)} If $x_0\in\XX$, the space $W_{\gamma,p}(x_0)$ comprises functions $f$ such that there exists $r>0$ with the property that for every $\phi\in C^\infty(\mathbb{B}(x_0,r))$, $\phi f\in W_{\gamma,p}$. 
If $A\subset \XX$, the space $W_{\gamma,p}(A)=\displaystyle\bigcap_{x_0\in A}W_{\gamma,p}(x_0)$; i.e., $W_{\gamma,p}(A)$ comprises functions which are in $W_{\gamma,p}(x_0)$ for each $x_0\in A$.
\end{definition}

A central theme in approximation theory is to characterize the smoothness spaces $W_{\gamma,p}$ in terms of the degree of approximation from some spaces; in our case we consider $\Pi_n$'s. 

For this purpose, we define some localized kernels and operators.

The kernels are defined by
\be\label{eq:kerndef}
\Phi_{n}(H;x,y)= \sum_{m=0}^\infty H\left(\frac{\lambda_m}{n}\right)\phi_m(x)\phi_m(y),
\ee
where $H :\RR\to\RR$ is a compactly supported function.

The operators corresponding to the kernels $\Phi_n$ are defined by
\be\label{eq:opdef}
\sigma_{n}(H;f,x)= \int_{\mathbb{X}}\Phi_{n}(H;x,y)f(y)d\mu^*(y) =\sum_{k: \lambda_k<n}H\left(\frac{\lambda_k}{n}\right)\hat{f}(k)\phi_k(x),
\ee
where
\be\label{eq:fourcoeff}
\hat{f}(k)=\int_\XX f(y)\phi_k(y)d\mu^*(y).
\ee

The following proposition recalls an important property of these kernels. 
Proposition~\ref{prop:kernloc}  is proved in \cite{mauropap}, and more recently  in much greater generality in \cite[Theorem~4.3]{tauberian}.
\begin{proposition}\label{prop:kernloc}
Let  $S>q+1$ be an integer, $H:\mathbb{R}\to \mathbb{R}$ be an even, $S$ times continuously differentiable, compactly supported function. 
 Then for every $x,y\in \mathbb{X}$, $N>0$,
\begin{equation}\label{eq:kernlocest}
| \Phi_N(H;x,y)|\ls \frac{N^{q}}{\max(1, (Nd(x,y))^S)},
\end{equation}
where the constant may depend upon $H$ and $S$, but not on $N$, $x$, or $y$.
\end{proposition}

In the remainder of this paper, we fix a filter $h$; i.e., an infinitely differentiable function $h: [0,\infty)\to [0,1]$, such that $h(t)=1$ for $0\le t\le 1/2$, $h(t)=0$ for $t\ge 1$. 
The domain of the filter $h$ can be extended to $\RR$ by setting $h(-t)=h(t)$. 
Since $h$ is fixed, its mention will be omitted from the notation unless we feel that this would cause a confusion.
The following theorem gives a crucial property of the operators, proved in several papers of ours in different contexts, see \cite{mhaskar2020kernel} for a recent proof.
\begin{theorem}\label{theo:goodapprox}
Let $n>0$. If $P\in\Pi_{n/2}$, then $\sigma_n(P)=P$. Also,
for any $p$ with $1\le p\le\infty$, 
\be\label{opbd}
\|\sigma_n(f)\|_p \ls \|f\|_p, \qquad f\in L^p.
\ee 
If $1\le p\le \infty$, and $f\in L^p(\XX)$, then
\be\label{goodapprox}
E_n(p,f)\le \|f-\sigma_n(f)\|_{p,\mu^*}\ls E_{n/2}(p,f).
\ee
\end{theorem}

While Theorem~\ref{theo:goodapprox} gives, in particular, a characterization of the global smoothness spaces $W_{\gamma,p}$, the characterization of local smoothness requires two more assumptions: the partition of unity and product assumption. 

\begin{definition}[\textbf{Partition of unity}] We say that a set $X$ has a partition of unity if for every $r>0$, there exists a countable family $\mathcal{F}_r=\{\psi_{k,r}\}_{k=0}^\infty$ of $C^\infty$ functions with the following properties:
\begin{enumerate}
    \item Each $\psi_{k,r}\in \mathcal{F}_r$ is supported on $\mathbb{B}(x_k,r)$ for some $x_k\in X$.
    \item For every $\psi_{k,r}\in\mathcal{F}_r$ and $x\in X$, $0\leq \psi_{k,r}(x)\leq 1$.
    \item For every $x\in X$ there exists a finite subset $\mathcal{F}_r(x)\subseteq \mathcal{F}_r$ (with cardinality bounded independently of $x$) such that for all $y\in\mathbb{B}(x,r)$
    \begin{equation}
        \sum_{\psi_{k,r}\in\mathcal{F}_r(x)}\psi_{k,r}(y)=1.
    \end{equation}
\end{enumerate}
\end{definition}

\begin{definition}[\textbf{Product assumption}]\label{def:prod}
    We say that a data space $\Xi$ satisfies the product assumption if there exists $A^*\geq 1$ and a family $\{R_{j,k,n}\in \Pi_{A^*n}\}$ such that for every $S>0$,
    \begin{equation}
        \lim_{n\to \infty}n^S\left(\max_{\lambda_k,\lambda_j<n}\norm{\phi_k\phi_j-R_{j,k,n}}_\mathbb{X}\right)=0.
    \end{equation}
If instead for every $n>0$ and $P,Q\in \Pi_n$ we have $PQ\in \Pi_{A^*n}$, then we say that $\Xi$ satisfies the \textbf{strong product assumption}.
\end{definition}

In the most important manifold case, the partition of unity assumption is always satisfied \cite[Chapter~0, Theorem~5.6]{docarmo_riemannian}. 
It is shown in \cite{geller2011band, modlpmz} that the strong product assumption is satisfied if $\phi_k$'s are eigenfunctions of certain differential equations on a Riemannian manifold and the $\lambda_k$'s are the corresponding eigenvalues.
We do not know of any example where this property does not hold, yet cannot prove that it holds in general. 
Hence, we have listed it as an assumption.

Our characterization of local smoothness (\cite{compbio, heatkernframe, mhaskar2020kernel}) is the following.
\begin{theorem}\label{theo:paleywiener}
Let $1\le p\le \infty$, $\gamma>0$, $f\in X^p$, $x_0\in\XX$. We assume the partition of unity and the  product assumption. 
Then the following are equivalent.\\
{\rm (a)} $f\in W_{\gamma,p}(x_0)$.\\
{\rm (b)}  There exists a ball $\BB$ centered at $x_0$ such that
\be\label{eq:op_implies_sobol}
\sup_{n\ge 0}2^{n\gamma}\|f-\sigma_{2^n}(f)\|_{p,\mu^*,\BB} <\infty.
\ee
\end{theorem}

A direct corollary is the following.
\begin{corollary}\label{cor:set_loc_smooth}
Let $1\le p\le \infty$, $\gamma>0$, $f\in X^p$, $A$ be a compact subset of $\XX$. We assume the partition of unity and the product assumption. 
Then the following are equivalent.\\
{\rm (a)} $f\in W_{\gamma,p}(A)$.\\
{\rm (b)} There exists $r>0$ such that
\be\label{eq:op_implies_sobol_bis}
\sup_{n\ge 0}2^{n\gamma}\|f-\sigma_{2^n}(f)\|_{p,\mu^*,\BB(A,r)} <\infty.
\ee
\end{corollary}

\section{Joint data spaces}\label{bhag:jointspaces}

In order to motivate our definitions in this section, we first consider  a couple of examples.

\begin{uda}\label{uda:coifmanhirn}
{\rm
Let $\Xi_j=(\XX_j, d_j, \mu_j^*, \{\lambda_{j,k}\}_{k=0}^\infty, \{\phi_{j,k}\}_{k=0}^\infty)$, $j=1,2$ be two data spaces with exponent $q$. 
We denote the heat kernel in each case by
$$
K_{j,t}(x,y)=\sum_{k=0}^\infty \exp(-\lambda_{j,k}^2t)\phi_{j,k}(x)\phi_{j,k}(y), \qquad j=1,2, \ x,y\in \XX, \ t>0,
$$
In the paper \cite{coifmanhirn}, Coifman and Hirn assumed that $\XX_1=\XX_2=\XX$, $\mu_1^*=\mu_2^*=\mu^*$, and proposed the diffusion distance between points $x_1, x_2$ to be the square root of
$$
K_{1,2t}(x_1,x_2)+K_{2,2t}(x_2,x_2)-2\int_\XX K_{1,t}(x_1,y)K_{2,t}(y,x_2)d\mu^*(y).
$$
Writing, in this example only, 
\be\label{eq:coif_hirn_conn}
A_{j,k}=\int_\XX \phi_{1,j}(y)\phi_{2,k}(y)d\mu^*(y),
\ee
we get 
\be\label{eq:coif_hirn_heat}
\int_\XX K_{1,t}(x_1,y)K_{2,t}(y,x_2)d\mu^*(y)=\sum_{j,k}\exp\left(-(\lambda_{1,j}^2+\lambda_{2,k}^2)t\right)A_{j,k}\phi_{1,j}(x_1)\phi_{2,k}(x_2).
\ee
Furthermore,  the Gaussian upper bound conditions imply that
\be\label{eq:coif_hirn_joint}
\begin{aligned}
\int_\XX K_{1,t}(x_1,y)K_{2,t}(y,x_2)d\mu^*(y)&\ls t^{-q}\int_\XX \exp\left(-c\frac{d_1(x_1,y)^2+d_2(y,x_2)^2}{t}\right)d\mu^*(y)\\
&\ls t^{-q}\exp\left(-c\frac{\left(\min_{y\in\XX}\left(d_1(x_1,y)+d_2(y,x_2)\right)\right)^2}{t}\right).
\end{aligned}
\ee
Writing, in this example only, 
$$
d_{1,2}(x_1,x_2)= \min_{y\in\XX}\left(d_1(x_1,y)+d_2(y,x_2)\right)=d_{2,1}(x_2,x_1),
$$
we observe that for any $x_1,x_1', x_2,x_2'\in\XX$,
$$
\begin{aligned}
d_{1,2}(x_1,x_2)&\le d_{1,2}(x_1',x_2)+d_1(x_1,x_1'),\\
d_{1,2}(x_1,x_2)&\le d_{1,2}(x_1,x_2')+d_1(x_2,x_2').
\end{aligned}
$$
\qed}
\end{uda}

\begin{uda}\label{uda:jacobispace2}
{\rm
In this example we let $\alpha_i, \beta_i \ge -1/2$ for $i=1,2$ and assume that $a=\abs{\alpha_1-\alpha_2}/2,b=\abs{\beta_1-\beta_2}/2\in \mathbb{N}$. Then we select the following two data spaces as defined in Example~\ref{uda:jacobispace}
\begin{equation}
\Xi_i=([0,\pi],d_i,\frac{1}{\pi}d\theta,\{\lambda_{i,n}\},\{\sqrt{\pi}\phi^{(\alpha_i,\beta_i)}_n\}).
\end{equation}
Since both spaces already have the same distance, we will define a joint distance for the systems accordingly:
\begin{equation}
d_{1,2}(\theta_1,\theta_2)=d_1(\theta_1,\theta_2)=d_2(\theta_1,\theta_2)=|\theta_1-\theta_2|.
\end{equation}
Similar to Example~\ref{uda:coifmanhirn} above, we are considering two data spaces with the same underlying space and measure. However, we now proceed in a different manner. Let us denote
\begin{equation}
\varOmega(\theta)=(1-\cos\theta)^{a}(1+\cos\theta)^{b}.
\end{equation}
Let $\overline{\alpha}=\max(\alpha_1,\alpha_2)$ and $\overline{\beta}=\max(\beta_1,\beta_2)$. Then we define
\begin{equation}
\begin{aligned}
A_{m,n}=&\int_{0}^\pi \phi_m^{(\alpha_1,\beta_1)}(\theta)\phi_n^{(\alpha_2,\beta_2)}(\theta)\varOmega(\theta)d\theta\\
=&\int_0^\pi p_m^{(\alpha_1,\beta_1)}(\cos\theta)p_n^{(\alpha_2,\beta_2)}(\cos\theta)(1-\cos\theta)^{\overline{\alpha}+1/2}(1+\cos\theta)^{\overline{\beta}+1/2}d\theta.
\end{aligned}
\end{equation}
The orthogonality of the Jacobi polynomials tells us that $A_{m,n}=0$ at least when $m>n+2a+2b$ or $n>m+2a+2b$. Furthermore, we have the following two sums
\begin{equation}\label{eq:orthosums}
\sum_{n}A_{m,n}\phi_{n}^{(\alpha_2,\beta_2)}(\theta)=\varOmega(\theta)\phi_{m}^{(\alpha_1,\beta_1)}(\theta),\hspace{5pt}\sum_{m}A_{m,n}\phi_{m}^{(\alpha_1,\beta_1)}(\theta)=\varOmega(\theta)\phi_{n}^{(\alpha_2,\beta_2)}(\theta).
\end{equation}
We define $\ell_{m.n}=\sqrt{\lambda_{1,m}^2+\lambda_{2,n}^2}$, utilize the Gaussian upper bound property for $\Xi_i$ and Equation~\eqref{eq:orthosums} to deduce as in Example~\ref{uda:coifmanhirn} that
\begin{equation}
\begin{aligned}
&\abs{\pi\sum_{m,n}\exp\left(-\ell_{m,n}^2t\right)A_{m,n}\phi_m^{(\alpha_1,\beta_1)}(\theta_1)\phi_n^{(\alpha_2,\beta_2)}(\theta_2)}\\
=&\abs{\int_0^\pi K_{1,t}(\theta_1,\phi)K_{2,t}(\phi,\theta_2)\varOmega(\phi)d\phi}\\
\lesssim& t^{-1}\exp\left(-c\frac{d_{1,2}(\theta_1,\theta_2)^2}{t}\right).
\end{aligned}
\end{equation}
We note (cf.  \cite[Lemma 5.2]{mhaskar2020kernel}) that
\begin{equation}\begin{aligned}
&\pi\sum_{m,n:\ell_{m,n}<N}\abs{A_{m,n}\phi_m^{(\alpha_1,\beta_1)}(\theta_1)\phi_n^{(\alpha_2,\beta_2)}(\theta_2)}\\
\lesssim& \norm{\varOmega}_{[0,\pi]}\sum_{m:\lambda_{1,m}<N}\abs{\phi_m^{(\alpha_1,\beta_1)}(\theta_1)\phi_m^{(\alpha_1,\beta_1)}(\theta_2)}\\
\lesssim& N.
\end{aligned}
\end{equation}
\qed}
\end{uda}

Motivated by these examples, we now give a series of definitions, culminating in Definition~\ref{def:jointspace}.
First, we define the notion of a joint distance.

\begin{definition}\label{def:jointdist}
Let $\XX_1$, $\XX_2$ be  metric  spaces, with each $\XX_j$ having a metric $d_j$.
 A function $d_{1,2}:\mathbb{X}_1\times \mathbb{X}_2\to [0,\infty)$ will be called a \textbf{joint distance} if the following \textbf{generalized triangle inequalities} are satisfied for $x_1,x_1'\in\mathbb{X}_1$ and $x_2,x_2'\in\mathbb{X}_2$:
\be\label{eq:joint_triangle}
\begin{aligned}
d_{1,2}(x_1,x_2) &\le d_1(x_1,x_1')+d_{1,2}(x_1',x_2),\\
d_{1,2}(x_1,x_2) &\le  d_{1,2}(x_1,x_2')+d_{2}(x_2',x_2).\\
\end{aligned}
\ee
\end{definition}

For convenience of notation we denote $d_{2,1}(x_2,x_1)= d_{1,2}(x_1,x_2)$. 
Then for $r>0$, $x_1\in\mathbb{X}_1$, $x_2\in \mathbb{X}_2$, $A_1\subset \XX_1$, $A_2\subset \XX_2$, we  define
\begin{align*}
    \mathbb{B}_1(x_1,r)&=\{z\in \mathbb{X}_1:d_1(x_1,z)\leq r\}, &\mathbb{B}_2(x_2,r)&=\{z\in \mathbb{X}_2:d_2(x_2,z)\leq r\},\\
    \mathbb{B}_{1,2}(x_1,r)&=\{z\in \mathbb{X}_2:d_{1,2}(x_1,z)\leq r\}, &\mathbb{B}_{2,1}(x_2,r)&=\{z\in \mathbb{X}_1:d_{2,1}(x_2,z)\leq r\},
 \end{align*}
\be\label{eq:jointballs}
\begin{aligned}
    d_{1,2}(A_1,x_2)=&\inf_{x\in A_1\subseteq \mathbb{X}_1}d_{1,2}(x,x_2),\\
    d_{1,2}(x_1,A_2)=d_{2,1}(A_2,x_1)=&\inf_{y\in A_2\subseteq \mathbb{X}_2}d_{2,1}(y,x_1).
\end{aligned}
\ee
We recall here that an infimum over an empty set is defined to be $\infty$. 

\begin{definition}\label{def:jointheatkern}
Let $\mathbf{A}=(A_{j,k})_{j,k=0}^\infty$ (\textbf{connection coefficients}) and $\mathbf{L}=(\ell_{j,k})_{j,k=0}^\infty$ (\textbf{joint eigenvalues}) be bi-infinite matrices. For $x_1\in\XX_1$, $x_2\in\XX_2$, $t>0$, the \textbf{joint heat kernel} is defined formally by
\be\label{eq:jointheatkerndef}
\begin{aligned}
K_t(\Xi_1,\Xi_2;x_1,x_2)=&K_t(\Xi_1,\Xi_2;\mathbf{A}, \mathbf{L}; x_1,x_2)\\
=&\sum_{j,k=0}^\infty \exp(-\ell_{j,k}^2 t)A_{j,k}\phi_{1,j}(x_1)\phi_{2,k}(x_2)\\
=&\lim_{n\to\infty}\sum_{j,k:\ell_{j,k}<n} \exp(-\ell_{j,k}^2 t)A_{j,k}\phi_{1,j}(x_1)\phi_{2,k}(x_2).
\end{aligned}
\ee
\end{definition}
\begin{definition}\label{def:jointspace}
For $m=1,2$, let $\Xi_m=\left(\XX_m, d_m,\mu_m^*, \{\lambda_{m,k}\}_{k=0}^\infty, \{\phi_{m,k}\}_{k=0}^\infty\right) $ be compact data spaces.
With the notation above, assume each $\ell_{j,k}\ge 0$ and that for any $u>0$, the set $\{(j,k): \ell_{j,k}<u\}$ is finite.
A \textbf{joint (compact) data space} $\Xi$ is  a tuple 
$$
(\Xi_1,  \Xi_2, d_{1,2}, \mathbf{A}, \mathbf{L}),
$$
where each of the following conditions is satisfied for some $Q>0$:
\begin{enumerate}
\item (\textbf{Joint regularity})
There exist $q_1, q_2>0$ such that
\be\label{eq:ballmeasurecond1}
\mu_1^*(\BB_{2,1}(x_2,r))\le cr^{q_1}, \quad \mu_2^*(\BB_{1,2}(x_1,r))\le cr^{q_2}, \qquad x_1\in\XX_1,\ x_2\in\XX_2,\  r>0.
\ee
\item (\textbf{Variation bound)}
For each $n>0$,
\begin{equation}\label{eq:variation}
    \sum_{j,k:\ell_{j,k}<n}\abs{A_{j,k}\phi_{1,j}(x_1)\phi_{2,k}(x_2)}\lesssim n^Q, \qquad x_1\in\XX_1, \ x_2\in\XX_2.
\end{equation}
\item (\textbf{Joint Gaussian upper bound})
The limit in \eref{eq:jointheatkerndef} exists for all $x_1\in\XX_1$, $x_2\in\XX_2$, and
\be\label{eq:jointoffdiagbd}
|K_t(\Xi_1,\Xi_2;x_1,x_2)| \le c_1t^{-c}\exp\left(-c_2\frac{d_{1,2}(x_1,x_2)^2}{t}\right), 
\qquad x_1\in\XX_1,\ x_2\in\XX_2.
\ee
\end{enumerate}
We refer to $(Q, q_1, q_2)$ as the \textbf{(joint) exponents} of the joint data space.
\end{definition}

The kernel corresponding to the one defined in Equation~\eqref{eq:kerndef} is the following, where $H: [0,\infty)\to[0,\infty)$ is a compactly supported function.
\begin{equation}\label{eq:jointkernel}
\Phi_n(H,\Xi_1,\Xi_2; x_1,x_2)=\sum_{j,k=0}^\infty H\left(\frac{\ell_{j,k}}{n}\right)A_{j,k}\phi_{1,j}(x_1)\phi_{2,k}(x_2).
\end{equation}
For $f\in L^1(\mu^*_2)+L^\infty(\mu^*_2)$ and $x_1\in \mathbb{X}_1$, we also define
\begin{equation}\begin{aligned}\label{eq:jointopdef}
\sigma_{n}(H, \Xi_1,\Xi_2;f)(x_1)=&\int_{\mathbb{X}_2}f(x_2)\Phi_n(H,\Xi_1,\Xi_2;x_1,x_2)d\mu^*_2(x_2)\\
=&\sum_{j,k=0}^\infty H\left(\frac{\ell_{j,k}}{n}\right)A_{j,k}\hat{f}(\Xi_2;k)\phi_{1,j}(x_1).
\end{aligned}
\end{equation}

The localization property of the kernels is given in the following proposition (cf. \cite[Eqn.~(4.5)]{tauberian}).
\begin{proposition}\label{prop:jointkernloc}
Let  $S>Q+1$ be an integer, $H:\mathbb{R}\to \mathbb{R}$ be an even, $S$ times continuously differentiable, compactly supported function. 
 Then for every $x_1\in \mathbb{X}_1$, $x_2\in \mathbb{X}_2$, $N>0$,
\begin{equation}\label{eq:jointkernlocest}
| \Phi_N(H,\Xi_1,\Xi_2; x_1,x_2)|\ls \frac{N^Q}{\max(1, (Nd_{1,2}(x_1,x_2))^S)},
\end{equation}
where the constant involved may depend upon $H$,  and $S$, but not on $N$, $x_1$, $x_2$. 
\end{proposition}
In the sequel, we will fix $H$ to be the filter $h$ introduced in Section~\ref{bhag:singlespace}, and will omit its mention from all notations. 
Also, we take $S>\max(Q, q_1, q_2)+1$ to be fixed, although we may put additional conditions on $S$ as needed.
As before, all constants may depend upon $h$ and $S$.

In the remainder of this paper, we will take $p=\infty$, work only with continuous functions on $\XX_1$ or $\XX_2$, and use $\|f\|_K$ to denote the supremum norm of $f$ on a set $K$. Accordingly, we will omit the index $p$ from the notation for the smoothness classes; e.g., we will write $W_\gamma(\Xi_1;B)$ instead of $W_{\gamma,\infty}(\Xi_1;B)$. The results in the sequel are similar in the case where $p<\infty$ due to the Riesz-Thorin interpolation theorem, but more notationally exhausting without adding any apparent new insights.

We end the section with a condition on the operator defined in Equation~\eqref{eq:jointopdef} that is useful for our purposes.

\begin{definition}[\textbf{Polynomial preservation condition}]\label{def:polypres} Let  $(\Xi_1,\Xi_2,d_{1,2},\mathbf{A},\mathbf{L})$ be a joint data space. We say the \textbf{polynomial preservation condition} is satisfied if there exists some $c^*>0$ with the property that if $P_{n}\in \Pi_{n}(\Xi_2)$, then $\sigma_m(\Xi_1,\Xi_2;P_n)=\sigma_{c^*n}(\Xi_1,\Xi_2;P_n)$ for all $m\geq c^*n$.
\end{definition}

\begin{remark}\label{rem:polypres}
The polynomial preservation condition is satisfied if, for any $n>0$, we have the following inclusion:
\begin{equation}\label{eq:polyprescon}
\{(i,j):A_{i,j}\neq 0,\lambda_{2,j}<n\}\subseteq \{(i,j):\ell_{i,j}\leq c^*n,\lambda_{1,i}<c^*n\}.
\end{equation}
\end{remark}

\begin{uda}\label{uda:jacobispace3}
{\rm
We utilize the same notation as in Examples~\ref{uda:jacobispace}~and~\ref{uda:jacobispace2}. We now see, in light of Definition~\ref{def:jointspace}, that $(\Xi_1,\Xi_2,d_{1,2},\mathbf{A},\mathbf{L})$ is a joint data space with exponents $(1,1,1)$. It is clear that both the partition of unity and strong product assumption hold in these spaces. One may also recall that $A_{m,n}=0$ at least whenever $m>n+2a+2b$, so there exists $c^*$ such that Equation~\eqref{eq:polyprescon} is satisfied. As a result, we conclude the polynomial preservation condition holds.
\qed}
\end{uda}

\section{Local approximation in joint data spaces}\label{bhag:locapprox}
 In this section, we assume a fixed joint data space as in Section~\ref{bhag:jointspaces}.  
 We are interested in the following questions. 
 Suppose $f\in C(\XX_2)$, and we have information about $f$ only in the neighborhood of a compact set $A\subseteq\XX_2$.
 Under what conditions on $f$ and a subset $B\subseteq \XX_1$ can $f$ be lifted to a function $\mathcal{E}(f)$ on $B$? 
 Moreover, how does the local smoothness of $\mathcal{E}(f)$ on $B$ depends upon the local smoothness of $f$ on $A$? We now give definitions for $\mathcal{E}(f),A,B$ for which we have considered these questions.

 \begin{definition}\label{def:liftedfunction} Given $f\in C(\mathbb{X}_2)$, we define the \textbf{lifted function} $\mathcal{E}(f)$ to be the limit
 \begin{equation}\label{eq:ef}
\mathcal{E}(f)=\lim_{n\to\infty}\sigma_n(\Xi_1,\Xi_2;f),
 \end{equation}
 if the limit exists.
 \end{definition}

\begin{definition}\label{def:imageset}
 Let $r,s>0$ and $A\subseteq \mathbb{X}_2$ be a compact subset with the property that there exists a compact subset $B^{-}\subset \mathbb{X}_1$ such that
\begin{equation}\label{eq:rcon}
B^-\subseteq \{x_1: d_{1,2}(x_1,\mathbb{X}_2\setminus A)\geq s+r\}
\end{equation}
for some $r>0$. 
We then define the \textbf{image set} of $A$ by 
\be\label{eq:imageset}
\mathcal{I}(r,s;A)=\BB_1(B^-,s)=\{x_1:d_1(x_1,B^-)\leq s\}.
\ee
If the set $B^{-}$ does not exist, then we define $\mathcal{I}(r,s;A)=\emptyset$.
\end{definition}
\begin{remark}\label{rem:imageset}
{\rm 
In the sequel we fix $r, s>0$ and  a compact subset $A\subseteq \mathbb{X}_2$ such that $B^-$ defined in Equation~\eqref{eq:rcon} is nonempty. We write $B= \mathcal{I}(r,s;A)$. We note that, due to the generalized triangle inequality~\eqref{eq:joint_triangle}, we have the important property
\begin{equation}
B^-\varsubsetneq B\subseteq \{x_1: d_{1,2}(x_1,\mathbb{X}_2\setminus A)\geq r\}.
\end{equation}
\qed}
\end{remark}
We now state our main theorem. Although there is no explicit mention of $B^-$ in the statement of the theorem, Remark~\ref{rem:thmclarify} and Example~\ref{uda:jacobispace4} clarify the benefit of such a construction.


\begin{theorem}\label{theo:mainthm}
Let  $(\Xi_1,\Xi_2,d_{1,2},\mathbf{A},\mathbf{L})$ be a joint data space with exponents $(Q, q_1, q_2)$.
We assume that the polynomial preservation condition holds with parameter $c^*$. Suppose $\mathbb{X}_2$ has a partition of unity. \\
{\rm (a)} Let $f\in C(\mathbb{X}_2)$, satisfying 
\begin{equation}\label{eq:thmcon1}
\sum_{m=0}^\infty 2^{m(Q-q_2)}\norm{\sigma_{2^{m+1}}(\Xi_2;f)-\sigma_{2^m}(\Xi_2;f)}_A<\infty.
\end{equation}
Then $\mathcal{E}(f)$ as defined in Definition~\ref{def:liftedfunction} exists on $B$ and for $c^*r2^n\geq 1$ we have
\begin{equation}\label{eq:ef-sigma}
\begin{aligned}
\norm{\mathcal{E}(f)-\sigma_{c^*2^n}(\Xi_1,\Xi_2;f)}_{B}\lesssim& 2^{n(Q-q_2)}\norm{f-\sigma_{2^n}(\Xi_2;f)}_A+\norm{f}_{\mathbb{X}_2}2^{n(Q-S)}r^{q_2-S}\\
&+\sum_{m=n}^\infty 2^{m(Q-q_2)}\norm{\sigma_{2^{m+1}}(\Xi_2;f)-\sigma_{2^m}(\Xi_2;f)}_A.
\end{aligned}
\end{equation}
In particular, if $\Xi_1$ satisfies the strong product assumption, $\mathbb{X}_1$ has a partition of unity, and $\alpha>0$ is given such that $\alpha\ell_{j,k}\geq \lambda_{1,j}$ for all $j,k\in\mathbb{N}$, then $\sigma_{n}(\Xi_1,\Xi_2;f)\in \Pi_{\alpha n}(\Xi_1)$. \\[1ex]
{\rm (b)}
If additionally, $f\in W_{\gamma}(\Xi_2; A)$ with $Q-q_2<\gamma<S-q_2$, then $\mathcal{E}(f)$ is continuous on $B$ and for $\phi\in C^\infty(B)$, we have $\phi \mathcal{E}(f)\in W_{\gamma-Q+q_2}(\Xi_1)$.
\end{theorem}

\begin{remark}\label{rem:thmclarify}
{\rm
Given the assumptions of Theorem~\ref{theo:mainthm}, $\mathcal{E}(f)$ is not guaranteed to be continuous on the entirety of $\mathbb{X}_1$ (or even defined outside of $B$). As a result, in the setting of \ref{theo:mainthm}(b) we cannot say $\mathcal{E}(f)$ belongs to any of the smoothness classes defined in this paper. However we can still say, for instance, that 
\begin{equation}\label{eq:b-inf}
\inf_{P\in \Pi_{2^n}(\Xi_1)}\norm{\mathcal{E}(f)-P}_{B^-}\lesssim 2^{-n(\gamma-Q+q_2)}
\end{equation}
(this can be seen directly by taking $\phi\in C^\infty(\Xi_1)$ such that $\phi(x)=1$ when $x\in B^-$ and $\phi(x)=0$ when $x\in \mathbb{X}_1\setminus B$). Consequently, if it happens that $\mathcal{E}(f)\in C(\Xi_1)$, then $\mathcal{E}(f)\in W_{\gamma-Q+q_2}(\Xi_1;B^-)$.
\qed}
\end{remark}

\begin{uda}\label{uda:jacobispace4}
{\rm We now conclude the running examples from \ref{uda:jacobispace}, \ref{uda:jacobispace2}, and \ref{uda:jacobispace3} by demonstrating how one may utilize Theorem~\ref{theo:mainthm}. We assume the notation given in each of the prior examples listed.
First, we find the image set for
 $A=\mathbb{B}_2(\theta_0,r_0)$ given some $\theta_0\in [0,\pi]$ and $r_0>0$. We let $r=s=r_0/8$ in correspondence to Definition~\ref{def:imageset} and define
\begin{equation}\begin{aligned}
  B^-=&\mathbb{B}_1\left(\theta_0,\frac{3r_0}{4}\right)\\
	 =&\left\{\theta_1\in[0,\pi]: d_1(\theta_1,[0,\pi]\setminus \mathbb{B}_1(\theta_0,r_0))\geq \frac{r_0}{4}\right\}\\
=&\left\{\theta_1\in[0,\pi]: d_{1,2}(\theta_1,[0,\pi]\setminus A)\geq r+s\right\}.
\end{aligned}\end{equation}
Then we can let $B=\mathbb{B}_1\left(\theta_0,\frac{7r_0}{8}\right)=\mathbb{B}_1(B^-,r)$. By Theorem~\ref{theo:mainthm}(a), $f\in C([0,\pi])$ can be lifted to $\mathbb{B}_1\left(\theta_0,\frac{7r_0}{8}\right)$ (where we note that Equation~\eqref{eq:thmcon1} is automatically satisfied due to $Q=q_2=1$). Since $\ell_{m,n}= \lambda_{1,n}$, we have $\sigma_{n}(\Xi_1,\Xi_2;f)\in \Pi_{n}(\Xi_1)$. If we suppose $f\in W_\gamma(\Xi_2;A)$for some $\gamma>0$ (with $h$ chosen so $S$ is sufficiently large), then Theorem~\ref{theo:mainthm}(b) informs us that $\phi\mathcal{E}(f)\in W_{\gamma}(\Xi_1)$ for $\phi\in C^\infty(B)$. Lastly, as a result of Equation~\eqref{eq:b-inf}, we can conclude
\begin{equation}
\inf_{P\in \Pi_{2^n}(\Xi_1)}\norm{\mathcal{E}(f)-P}_{\mathbb{B}_1\left(\theta_0,\frac{3r_0}{r}\right)}\lesssim 2^{-n\gamma}.
\end{equation}
\qed}
\end{uda}

\section{Proofs}\label{bhag:proofs}
 In this section, we give a proof of Theorem~\ref{theo:mainthm} after proving some preperatory results. We assume that $(\Xi_1,\Xi_2,d_{1,2}, \mathbf{A}, \mathbf{L})$ is a joint data space with exponents $Q, q_1, q_2$.
 
\begin{lemma}\label{lemma:opbdlemma}
Let $x_1\in\XX_1$, $r>0$.
We have
\begin{equation}\label{eq:lockeraway}
    \int_{\mathbb{X}_2\setminus \mathbb{B}_{1,2}(x_1,r)}\abs{\Phi_n(\Xi_1,\Xi_2;x_1,x_2)}d\mu_2^*(x_2)\lesssim n^{Q-q_2}(\max(1, nr))^{q_2-S}.
\end{equation}
In particular, 
\begin{equation}\label{eq:globker}
    \int_{\mathbb{X}_2}\abs{\Phi_n(\Xi_1,\Xi_2;x_1,x_2)}d\mu_2^*(x_2)\lesssim n^{Q-q_2}.
\end{equation}
\end{lemma}

\begin{proof}\ 
In this proof only, define
\begin{equation}
    A_0=\mathbb{B}_{1,2}(x_1,r),\qquad A_m=\mathbb{B}_{1,2}(x_1,r2^m)\setminus \mathbb{B}_{1,2}(x_1,r2^{m-1})\text{ for all $m\in \mathbb{N}$.}
\end{equation}
Then the joint regularity condition \eqref{eq:ballmeasurecond1} implies
$
    \mu^*_2(A_m)\lesssim (r2^m)^{q_2},
$
for each $m$. 
We can also see by definition that when $x\in A_m$, then $d_{1,2}(x_1,x)>r2^{m-1}$. 
Since $S>q_2$,  we deduce that for $r n\ge 1$, 
\be\begin{aligned}\label{eq:pf1eqn1}
    \int_{\mathbb{X}_2\setminus A_0}\abs{\Phi_n(\Xi_1,\Xi_2; x_1,x_2)}d\mu_2^*(x_2)\lesssim& \sum_{m=1}^\infty \frac{n^Q\mu_2^*(A_m)}{(rn2^{m-1})^S}\\
    \lesssim& r^{q_2-S}n^{Q-S}\sum_{m=1}^\infty 2^{m(q_2-S)}\\
    \lesssim& n^{Q-q_2}(nr)^{q_2-S}.
\end{aligned}\ee
This completes the proof of \eqref{eq:lockeraway} when $nr\ge 1$.
The joint regularity condition and Proposition~\ref{prop:jointkernloc} show further that 
\be\label{eq:pf1eqn2}
\int_{A_0} \abs{\Phi_n(\Xi_1,\Xi_2;x_1,x_2)}d\mu_2^*(x_2)\ls n^Q\mu_2^*(A_0)\ls n^Qr^{q_2}=n^{Q-q_2}(nr)^{q_2}.
\ee
We use $r=1/n$ in the estimates \eqref{eq:pf1eqn1} and \eqref{eq:pf1eqn2} and add the estimates to arrive at both \eqref{eq:globker} and the case $r\le 1/n$ of \eqref{eq:lockeraway}.
\end{proof}

The next lemma gives a local bound on the kernels $\sigma_n$ defined in \eqref{eq:jointopdef}.

\begin{lemma}\label{lemma:jointop_loc_lemma}
Let $A$ and $B$ be as defined in Remark~\ref{rem:imageset}.
For a continuous $f:A\to\RR$, we have
\be\label{eq:joint_loc_opbd}
\|\sigma_n(\Xi_1,\Xi_2;f)\|_B \ls n^{Q-q_2}\left\{\|f\|_A +\|f\|_{\XX_2}(\max(1, nr))^{q_2-S}\right\}.
\ee
\end{lemma}

\begin{proof}\ 
Let $x_1\in B$.
In view of the joint triangle inequality \eqref{eq:joint_triangle}, we have $d_{1,2}(x_1,x_2)\ge r$ for all $x_2\in \XX_2\setminus A$.
Therefore, Lemma~\ref{lemma:opbdlemma} shows that
\begin{equation}
\begin{aligned}
|\sigma_n(\Xi_1,\Xi_2;f)(x_1)|\le& \int_{\XX_2} |f(x_2)\Phi_n(\Xi_1,\Xi_2;x_1,x_2)|d\mu_2^*(x_2)\\
=& \int_A |f(x_2)\Phi_n(\Xi_1,\Xi_2;x_1,x_2)|d\mu_2^*(x_2)\\
&+\int_{\XX_2\setminus A}|f(x_2)\Phi_n(\Xi_1,\Xi_2;x_1,x_2)|d\mu_2^*(x_2) \\
\ls& n^{Q-q_2}\|f\|_A +\|f\|_{\XX_2}\int_{\XX_2\setminus \BB_{1,2}(x_1,r)}|\Phi_n(\Xi_1,\Xi_2;x_1,x_2)|d\mu_2^*(x_2)\\
\ls& n^{Q-q_2}\left\{\|f\|_A +\|f\|_{\XX_2}(\max(1, nr))^{q_2-S}\right\}.
\end{aligned}
\end{equation}
\end{proof}

\begin{lemma}\label{lemma:E}
We assume  the polynomial preservation condition with parameter $c^*$. Let $f\in C(\XX_2)$ satisfy \eqref{eq:thmcon1}.
Then
\begin{equation}
    \mathcal{E}(f)= \lim_{n\to \infty}\sigma_{2^n}(\Xi_1,\Xi_2;f)
\end{equation}
exists on $B$. 
Furthermore, when $c^*2^n>1/r$, we have
\begin{equation}
\begin{aligned}
    &\norm{\mathcal{E}(f)-\mathcal{E}(\sigma_{2^n}(\Xi_2;f))}_B\\
    \lesssim& \sum_{m=n}^\infty 2^{m(Q-q_2)}\norm{f-\sigma_{2^m}(\Xi_2;f)}_A+\norm{f}_{\mathbb{X}_2}n^{Q-S}r^{q_2-S}.
\end{aligned}
\end{equation}
\end{lemma}

\begin{proof}\ 
In this proof only we denote $P_{n}=\sigma_{n}(\Xi_2;f)$. 
Since $P_n\in \Pi_{n}(\Xi_2)$, the condition \eqref{eq:polyprescon} implies that
\begin{equation}
\mathcal{E}(P_n)=\sigma_{c^*n}(\Xi_1,\Xi_2;P_n)=\lim_{k\to \infty}\sigma_{k}(\Xi_1,\Xi_2;P_k)
\end{equation}
is defined on $\mathbb{X}_1$.
Theorem~\ref{theo:goodapprox} and Lemma~\ref{lemma:jointop_loc_lemma} then imply that
\begin{equation}
\begin{aligned}
&\norm{\mathcal{E}(P_{2^{m+1}})-\mathcal{E}(P_{2^{m}})}_B\\
=&
\norm{\sigma_{c^*2^{m+1}}(\Xi_1,\Xi_2;P_{2^{m+1}})-\sigma_{c^*2^{m+1}}(\Xi_1,\Xi_2; P_{2^m})}_B\\
\lesssim& 2^{m(Q-q_2)}(\norm{P_{2^{m+1}}-P_{2^m}}_A+\norm{f}_{\mathbb{X}_2}(\max(1, 2^mr))^{q_2-S}).
\end{aligned}
\end{equation}
We conclude that
\begin{equation}
\begin{aligned}
&\norm{\mathcal{E}(P_1)}_B+\sum_{m=0}^\infty \norm{\mathcal{E}(P_{2^{m+1}})-\mathcal{E}(P_{2^m})}_B\\
\lesssim&\norm{P_1}+\sum_{m=0}^\infty 2^{m(Q-q_2)}\norm{P_{2^{m+1}}-P_{2^m}}_A\\
&+\norm{f}_{\mathbb{X}_2}\left(\sum_{c^*2^m\leq 1/r} 2^{m(Q-q_2)}+r^{q_2-S}\sum_{c^*2^m>1/r}2^{m(Q-S)}\right)<\infty.
 \end{aligned}
\end{equation}
Thus, 
\begin{equation}
\mathcal{E}(f)=\mathcal{E}(P_1)+\sum_{m=0}^\infty \left(\mathcal{E}(P_{2^{m+1}})-\mathcal{E}(P_{2^m})\right)
\end{equation}
is defined on $B$. In particular, when $c^*2^n\geq 1/r$ it follows
\begin{equation}
    \norm{\mathcal{E}(f)-\mathcal{E}(P_{2^n})}_B\leq \sum_{m=n}^\infty 2^{m(Q-q_2)}\norm{P_{2^{m+1}}-P_{2^m}}_A+\norm{f}_{\mathbb{X}_2}2^{n(Q-S)}r^{q_2-S}.
\end{equation}
\end{proof}


Now we give the proof of Theorem~\ref{theo:mainthm}.

\begin{proof}\ 
In this proof only denote $P_n=\sigma_{n}(\Xi_2;f)\in\Pi_n(\Xi_2)$. 
We can deduce from Theorem~\ref{theo:goodapprox} and Lemma~\ref{lemma:jointop_loc_lemma} that for $c^*r2^n\geq 1$,
\begin{equation}
\begin{aligned}\label{eq:pf2eqn1}
&\norm{\sigma_{c^*2^n}(\Xi_1,\Xi_2;f)-\sigma_{c^*2^{n}}(\Xi_1,\Xi_2;P_{2^n})}_{B}\\
\lesssim& 2^{n(Q-q_2)}(\norm{f-P_{2^n}}_A+\norm{f-P_{2^n}}_{\mathbb{X}_2}2^{n(q_2-S)}r^{q_2-S})\\
\lesssim& 2^{n(Q-q_2)}(\norm{f-P_{2^n}}_A+\norm{f}_{\mathbb{X}_2}2^{n(q_2-S)}r^{q_2-S}).
\end{aligned}
\end{equation}
The polynomial preservation condition (Definition~\ref{def:polypres}) gives us that
\begin{equation}
\norm{\sigma_{c^*2^n}(\Xi_1,\Xi_2;P_{2^n})-\mathcal{E}(P_{2^n})}_{B}=0.
\end{equation}
Then utilizing Equation~\eqref{eq:pf2eqn1} and Lemma~\ref{lemma:E}, we see
\begin{equation}\begin{aligned}
&\norm{\sigma_{c^*2^n}(\Xi_1,\Xi_2;f)-\mathcal{E}(f)}_{B}\\
\leq&\norm{\sigma_{c^*2^n}(\Xi_1,\Xi_2;f)-\sigma_{c^*2^n}(\Xi_1,\Xi_2;P_{2^n})}_{B}\\
&+\norm{\sigma_{c^*2^n}(\Xi_1,\Xi_2;P_{2^n})-\mathcal{E}(P_{2^n})}_{B}+\norm{\mathcal{E}(P_{2^n})-\mathcal{E}(f)}_{B}\\
\lesssim& 2^{n(Q-q_2)}\norm{f-P_{2^n}}_A+\sum_{m=n}^\infty 2^{m(Q-q_2)}\norm{P_{2^{m+1}}-P_{2^m}}_A+\norm{f}_{\mathbb{X}_2}2^{n(Q-S)}r^{q_2-S}.
\end{aligned}
\end{equation}
This proves Equation~\eqref{eq:ef-sigma}.

In particular, when $\alpha\ell_{j,k}\geq \lambda_{1,j}$ and $\alpha>0$, the only $\phi_{1,j}(x_1)$ with non-zero coefficients in Equation~\eqref{eq:jointopdef} are those where $\ell_{j,k}<n$, which implies $\lambda_{1,j}<\alpha n$ and further that $\sigma_{n}(\Xi_1,\Xi_2;f)\in \Pi_{\alpha n}(\Xi_1)$.
This completes the proof of part (a).

In the proof of part (b),  we may assume without loss of generality that $\|f\|_{W_{\gamma}(\Xi_2;A)}+\norm{f}_{\mathbb{X}_2}=1$. We can see from Corollary~\ref{cor:set_loc_smooth} that for each $m$
\begin{equation}
    \norm{P_{2^{m+1}}-P_{2^m}}_A\leq \norm{P_{2^{m+1}}-f}_A+\norm{P_{2^{m}}-f}_A\lesssim 2^{-m\gamma},
\end{equation}
which implies that whenever $Q-q_2<\gamma$ we have
\begin{equation}
    \sum_{m=n}^\infty 2^{m(Q-q_2)}\norm{P_{2^{m+1}}-P_{2^m}}_A\lesssim 2^{n(Q-q_2-\gamma)}.
\end{equation}
Further, the assumption that $\gamma<S-q_2$ gives us
\begin{equation}
    2^{n(Q-S)}\lesssim 2^{n(Q-q_2-\gamma)}.
\end{equation}
Since $f\in W_\gamma(\Xi_2;A)$, we have from Corollary~\ref{cor:set_loc_smooth} that
\begin{equation}
\norm{f-P_{2^n}}_A\lesssim 2^{-n\gamma}.
\end{equation}
Using Equation~\eqref{eq:ef-sigma} from part (a), we see
\begin{equation}\label{eq:ef-sigmagamma}
\norm{\mathcal{E}(f)-\sigma_{c^*2^n}(\Xi_1,\Xi_2;f)}_{B}\lesssim (1+r^{q_2-S}\norm{f}_{\mathbb{X}_2})2^{n(Q-q_2-\gamma)}.
\end{equation}
Thus, $\{\sigma_{c^*2^n}(\Xi_1,\Xi_2;f)\}$ is a sequence of continuous functions converging uniformly to $\mathcal{E}(f)$ on $B$, so $\mathcal{E}(f)$ itself is continuous on $B$.
Let us define $R_{c^*\alpha 2^n}\in \Pi_{c^*\alpha 2^n}$ for each n such that $\norm{R_{c^*\alpha2^n}-\phi}_{\mathbb{X}_1}\lesssim 2^{-n\gamma}$. Theorem~\ref{theo:goodapprox} and the strong product assumption (Definition~\ref{def:prod}) allow us to write
\begin{equation}\label{eq:goodapproxapp}
    \sigma_{c^*A^*\alpha 2^{n+1}}(\Xi_1;R_{c^*\alpha2^n}\sigma_{c^*2^n}(\Xi_1,\Xi_2;f))=R_{c^*\alpha2^n}\sigma_{c^*2^n}(\Xi_1,\Xi_2;f).
\end{equation}
Using Equations~\eqref{eq:globker}~and~\eqref{eq:goodapproxapp}, Theorem~\ref{theo:goodapprox}, and the fact $\phi$ is supported on $B$, we can deduce
\begin{equation}\label{eq:phiRnpolybound}
\begin{aligned}
&\norm{\sigma_{c^*A^*\alpha 2^{n+1}}\big(\Xi_1;R_{c^*\alpha 2^n}\sigma_{c^*2^n}(\Xi_1,\Xi_2;f)-\phi\mathcal{E}(f)\big)}_{\mathbb{X}_1}\\
\lesssim& \norm{R_{c^*\alpha 2^n}\sigma_{c^*2^n}(\Xi_1,\Xi_2;f)-\phi\mathcal{E}(f)}_{\mathbb{X}_1}\\
\lesssim& \norm{\phi\mathcal{E}(f)-\phi\sigma_{c^*2^n}(\Xi_1,\Xi_2;f)}_{\mathbb{X}_1}+\norm{R_{c^*\alpha 2^n}-\phi}_{\mathbb{X}_1}\norm{\sigma_{c^*2^n}(\Xi_1,\Xi_2;f)}_{\mathbb{X}_1}\\
\lesssim&\norm{\mathcal{E}(f)-\sigma_{c^*2^n}(\Xi_1,\Xi_2;f)}_B+2^{n(Q-q_2-\gamma)}\norm{f}_{\mathbb{X}_2}.
\end{aligned}
\end{equation}
In view of Equations~\eqref{eq:ef-sigmagamma}~and~\eqref{eq:phiRnpolybound}, we can conclude that
\begin{equation}\begin{aligned}
&E_{c^*A^*\alpha 2^{n+1}}(\Xi_1,\phi\mathcal{E}(f))\\
    \lesssim&\norm{\phi \mathcal{E}(f)-\sigma_{c^*A^*\alpha 2^{n+1}}(\Xi_1,\phi \mathcal{E}(f))}_{\mathbb{X}_1}\\
    \leq&\norm{\phi \mathcal{E}(f)-R_{c^*\alpha 2^n}\sigma_{c^*2^n}(\Xi_1,\Xi_2;f)}_{\mathbb{X}_1}\\
    &+\norm{\sigma_{c^*A^*\alpha 2^{n+1}}\big(\Xi_1;R_{c^*\alpha 2^n}\sigma_{c^*2^n}(\Xi_1,\Xi_2;f)-\phi\mathcal{E}(f)\big)}_{\mathbb{X}_1}\\
    \lesssim& \norm{\mathcal{E}(f)-\sigma_{c^*2^n}(\Xi_1,\Xi_2;f)}_{B}+\norm{f}_{\mathbb{X}_2}2^{n(Q-q_2-\gamma)}\\
\lesssim&(1+\norm{f}_{\mathbb{X}_2}(1+r^{q_2-S})) 2^{n(Q-q_2-\gamma)}.
\end{aligned}\end{equation}
Thus, $\phi\mathcal{E}(f)\in W_{\gamma-Q+q_2}(\Xi_1)$, completing the proof of part (b).
\end{proof}

\section*{Appendix}

\subsection{Tauberian theorem}\label{bhag:tauberian}
For the convenience of the reader, we reproduce the Tauberian theorem from \cite[Theorem~4.3]{tauberian}.

We recall that if $\mu$ is an extended complex valued Borel measure on $\RR$, then its total variation measure is defined for a Borel set $B$ by
$$
|\mu|(B)=\sup\sum |\mu(B_k)|,
$$
where the sum is over a partition $\{B_k\}$ of $B$ comprising Borel sets, and the supremum is over all such partitions.

A measure $\mu$ on $\RR$ is called an even measure if $\mu((-u,u))=2\mu([0,u))$ for all $u>0$, and $\mu(\{0\})=0$. If $\mu$ is an extended complex valued measure on $[0,\infty)$, and $\mu(\{0\})=0$, we define a measure $\mu_e$ on $\RR$ by 
$$
\mu_e(B)=\mu\left(\{|x| : x\in B\}\right),
$$
and observe that $\mu_e$ is an even measure such that $\mu_e(B)=\mu(B)$ for $B\subset [0,\infty)$. In the sequel, we will assume that all measures on $[0,\infty)$ which do not associate a nonzero mass with the point $0$ are extended in this way, and will abuse the notation $\mu$ also to denote the measure $\mu_e$. In the sequel, the phrase ``measure on $\RR$'' will refer to an extended complex valued Borel measure having bounded total variation on compact intervals in $\RR$, and similarly for measures on $[0,\infty)$.

Our main Tauberian theorem is the following.

\begin{theorem}\label{theo:maintaubertheo}
Let $\mu$ be an extended complex valued measure on $[0,\infty)$, and $\mu(\{0\})=0$. We assume that there exist $Q, r>0$, such that each of the following conditions are satisfied.
\begin{enumerate}
\item 
\be\label{eq:muchristbd}
\tn\mu\tn_Q:=\sup_{u\in [0,\infty)}\frac{|\mu|([0,u))}{(u+2)^Q} <\infty,
\ee
\item There are constants $c, C >0$,  such that
\be\label{eq:muheatgaussbd}
\left|\int_\RR \exp(-u^2t)d\mu(u)\right|\le c_1t^{-C}\exp(-r^2/t)\tn\mu\tn_Q, \qquad 0<t\le 1.
\ee 
\end{enumerate}
Let $H:[0,\infty)\to\RR$, $S>Q+1$ be an integer, and suppose that there exists a measure $H^{[S]}$ such that
\be\label{eq:Hbvcondnew}
H(u)=\int_0^\infty (v^2-u^2)_+^{S}dH^{[S]}(v), \qquad u\in\RR,
\ee
and
\be\label{eq:Hbvintbdnew}
V_{Q,S}(H)=\max\left(\int_0^\infty (v+2)^Qv^{2S}d|H^{[S]}|(v), \int_0^\infty (v+2)^Qv^Sd|H^{[S]}|(v)\right)<\infty.
\ee
Then for $n\ge 1$,
\be\label{eq:genlockernest}
\left|\int_0^\infty H(u/n)d\mu(u)\right| \le c\frac{n^Q}{\max(1, (nr)^S)}V_{Q,S}(H)\tn\mu\tn_Q.
\ee
\end{theorem}

Proposition~\ref{prop:kernloc} is proved using this theorem with 
$$
\mu(u)=\mu_{x,y}(u)=\sum_{k : \lambda_k<u}\phi_k(x)\phi_k(y).
$$
Proposition~\ref{prop:jointkernloc} is proved using this theorem with 
$$
\mu(u)=\mu_{x_1,x_2}(u)=\sum_{j,k : \ell_{j,k}<u}A_{j,k}\phi_{1,j}(x_1)\phi_{2,k}(x_2).
$$

\subsection{Jacobi polynomials}\label{bhag:jacobi}

For $\alpha, \beta>-1$, $x\in (-1,1)$ and integer $\ell\ge 0$, the Jacobi polynomials $p_\ell^{(\alpha,\beta)}$ are defined by the Rodrigues' formula \cite[Formulas~(4.3.1), (4.3.4)]{szego}
\be\label{eq:rodrigues}
\begin{aligned}
&(1-x)^\alpha(1+x)^\beta p_\ell^{(\alpha,\beta)}(x)\\
=&\left\{\frac{2\ell+\alpha+\beta+1}{2^{\alpha+\beta+1}}\frac{\ell!(\ell+\alpha+\beta)!}{(\ell+\alpha)!(\ell+\beta)!}\right\}^{1/2}\frac{(-1)^\ell}{2^\ell \ell!}\frac{d^\ell}{dx^\ell}\left((1-x)^{\ell+\alpha}(1+x)^{\ell+\beta}\right),
\end{aligned}
\ee
where $z!$ denotes $\Gamma(z+1)$. The Jacobi polynomials satisfy the following well-known differential equation:
\begin{equation}\label{eq:jacobidifeq}
	{p''_n}^{(\alpha,\beta)}(x)(1-x^2)+(\beta-\alpha-(\alpha+\beta+2)x){p'_n}^{(\alpha,\beta)}(x)=-n(n+\alpha+\beta+1)p_n^{(\alpha,\beta)}(x).
\end{equation}
Each $p_\ell^{(\alpha,\beta)}$ is a polynomial of degree $\ell$ with positive leading coefficient, satisfying the orthogonality relation
\be\label{eq:jacobiortho}
\int_{-1}^1 p_\ell^{(\alpha,\beta)}(x)p_j^{(\alpha,\beta)}(x)(1-x)^\alpha(1+x)^\beta=\delta_{\ell,j},
\ee
and
\be\label{eq:pkat1}
p_\ell^{(\alpha,\beta)}(1)=\left\{\frac{2\ell+\alpha+\beta+1}{2^{\alpha+\beta+1}}\frac{\ell!(\ell+\alpha+\beta)!}{(\ell+\alpha)!(\ell+\beta)!}\right\}^{1/2}\frac{(\ell+\alpha)!}{\alpha!\ell!} \sim \ell^{\alpha+1/2}.
\ee
It follows that $p_\ell^{(\alpha,\beta)}(-x)=(-1)^\ell p_\ell^{(\beta,\alpha)}(x)$. 
In particular,  $p_{2\ell}^{(\alpha,\alpha)}$ is an even polynomial, and $p_{2\ell+1}^{(\alpha,\alpha)}$ is an odd polynomial.  
We note (cf. \cite[Theorem~4.1]{szego}) that
\begin{equation}\label{eq:evenjacobi}
\begin{aligned}
p_{2\ell}^{(\alpha,\alpha)}(x)=&2^{\alpha/2+1/4}p_\ell^{(\alpha,-1/2)}(2x^2-1)=2^{\alpha/2+1/4}(-1)^\ell p_\ell^{(-1/2,\alpha)}(1-2x^2)\\
p_{2\ell+1}^{(\alpha,\alpha)}(x)=&2^{\alpha/2+1/2}xp_\ell^{(\alpha,1/2)}(2x^2-1)=2^{\alpha/2+1/2}(-1)^\ell xp_\ell^{(1/2,\alpha)}(1-2x^2).
\end{aligned}\ee
It is known \cite{nowak2011sharp} that for $\alpha,\beta\ge -1/2$ and $\theta, \phi\in [0,\pi]$,
\be\begin{aligned}\label{eq:jacobigauss}
&\sum_{j=0}^\infty \exp(-j(j+\alpha+\beta+1)t)p_j^{(\alpha,\beta)}(\cos\theta)p_j^{(\alpha,\beta)}(\cos\phi)\\
\ls& (t+\theta\phi)^{-\alpha-1/2}(t+(\pi-\theta)(\pi-\phi))^{-\beta-1/2}t^{-1/2}\exp\left(-c\frac{(\theta-\phi)^2}{t}\right).
\end{aligned}\ee
We note that when $\beta=-1/2$, this yields
\be\begin{aligned}\label{eq:specialjacobigauss}
&\sum_{j=0}^\infty \exp(-j(j+\alpha+1/2)t)p_j^{(\alpha,-1/2)}(\cos\theta)p_j^{(\alpha,-1/2)}(\cos\phi)\\
\ls& t^{-\alpha-1}\exp\left(-c\frac{(\theta-\phi)^2}{t}\right).
\end{aligned}\ee
If $\ell\ge 0$ is an integer, and $\{Y_{\ell,k}\}$ is an orthonormal basis for the space $\HH_\ell^q$ of restrictions to the sphere $\SS^q$ (with respect to the probability volume measure) of $(q+1)$-variate homogeneous harmonic polynomials of total degree $\ell$, then
one has the
well-known addition formula \cite{mullerbk} and \cite[Chapter XI, Theorem 4]{batemanvol2} connecting $Y_{\ell,k}$'s with Jacobi polynomials defined in \eref{eq:rodrigues}:
\begin{equation}
\label{eq:addformula}
 \sum_{k=1}^{d\,^q_\ell} Y_{\ell,k}(\x)\overline{Y_{\ell,k}(\y)} =
\frac{\omega_q}{\omega_{q-1}} p_\ell^{(q/2-1,q/2-1)}(1)p_\ell^{(q/2-1,q/2-1)}(\x\cdot\y), \quad
\ell=0,1,\cdots,
\end{equation}
where $\omega_q=\operatorname{Vol}(\mathbb{S}^q)$ and $d\,^q_\ell=\operatorname{dim}(\mathbb{H}_\ell^q)$.

\end{document}